\documentclass{article}

\usepackage{microtype}
\usepackage{graphicx}
\usepackage{subfigure}
\usepackage{booktabs} 

\usepackage{hyperref}

\usepackage[accepted]{icml2024}

\usepackage{amsmath}
\usepackage{amssymb}
\usepackage{mathtools}
\usepackage{amsthm}

\usepackage[capitalize,noabbrev]{cleveref}

\theoremstyle{plain}
\newtheorem{theorem}{Theorem}[section]

\newtheorem{corollary}[theorem]{Corollary}
\theoremstyle{definition}
\newtheorem{definition}[theorem]{Definition}

\theoremstyle{remark}
\newtheorem{remark}[theorem]{Remark}

\usepackage[textsize=tiny]{todonotes}

\usepackage{amsmath}
\usepackage{amssymb}
\usepackage{mathtools}
\usepackage{amsthm}
\usepackage{bbm}
\usepackage{nicefrac}
\usepackage{wrapfig}

\usepackage{tabularx}
\usepackage[framemethod=TikZ]{mdframed}
\usepackage{xcolor}

\definecolor{lightgreen}{RGB}{144,238,144}
\definecolor{lightred}{RGB}{255,192,203}

\newcommand{\gbg}[1]{\colorbox{lightgreen}{\textcolor{black}{#1}}}
\newcommand{\rbg}[1]{\colorbox{lightred}{\textcolor{black}{#1}}}

\definecolor{light-gray}{gray}{0.95} 

\mdfdefinestyle{textbox}{%
    roundcorner=10pt,
    backgroundcolor=light-gray,
    linewidth=1pt,
    frametitlerule=true,
    frametitlebackgroundcolor=gray!20,
    frametitlefont=\sffamily\bfseries\color{black},
    align=center,
    userdefinedwidth=0.95\textwidth
}

\usepackage{amsmath,amsfonts,bm}

\def\eqref#1{equation~\ref{#1}}

\def\1{\bm{1}}

\def\vw{{\bm{w}}}
\def\vx{{\bm{x}}}

\def\vz{{\bm{z}}}

\def\mW{{\bm{W}}}

\DeclareMathAlphabet{\mathsfit}{\encodingdefault}{\sfdefault}{m}{sl}
\SetMathAlphabet{\mathsfit}{bold}{\encodingdefault}{\sfdefault}{bx}{n}

\def\gD{{\mathcal{D}}}

\def\gF{{\mathcal{F}}}

\def\gH{{\mathcal{H}}}

\def\gP{{\mathcal{P}}}

\def\gX{{\mathcal{X}}}

\def\gZ{{\mathcal{Z}}}

\def\sD{{\mathbb{D}}}

\def\sN{{\mathbb{N}}}

\newcommand{\R}{\mathbb{R}}

\DeclareMathOperator*{\argmax}{arg\,max}

\newcommand{\reals}{\R}
\newcommand{\naturals}{\sN}

\newcommand{\x}{\vx}
\newcommand{\z}{\vz}

\newcommand{\bw}{\vw}

\newcommand{\W}{\mW}

\newcommand{\cd}{\gD}
\newcommand{\ch}{\gH}
\newcommand{\cf}{\gF}
\newcommand{\cx}{\gX}

\newcommand{\cz}{\gZ}
\newcommand{\cp}{\gP}

\newcommand{\bbD}{\sD}

\newcommand{\inner}[1]{\left\langle #1\right\rangle}

\newcommand{\ind}{\mathbf{1}}

\newcommand{\Lin}{\mathrm{Lin}}
\newcommand{\poly}{\mathrm{poly}}

\icmltitlerunning{Auto-Regressive Next-Token Predictors are Universal Learners}

\begin{document}

\twocolumn[
\icmltitle{Auto-Regressive Next-Token Predictors are Universal Learners}




\begin{icmlauthorlist}
\icmlauthor{Eran Malach}{kempner}
\end{icmlauthorlist}

\icmlaffiliation{kempner}{Harvard University, Kempner Institute for the Study of Natural and Artificial Intelligence}

\icmlcorrespondingauthor{Eran Malach}{emalach@fas.harvard.edu}

\icmlkeywords{Machine Learning, ICML}

\vskip 0.3in
]



\printAffiliationsAndNotice{}  

\begin{abstract}
Large language models display remarkable capabilities in logical and mathematical reasoning, allowing them to solve complex tasks. Interestingly, these abilities emerge in networks trained on the simple task of next-token prediction. In this work, we present a theoretical framework for studying auto-regressive next-token predictors. We demonstrate that even simple models such as linear next-token predictors, trained on Chain-of-Thought (CoT) data, can approximate any function efficiently computed by a Turing machine. We introduce a new complexity measure---length complexity---which measures the number of intermediate tokens in a CoT sequence required to approximate some target function, and analyze the interplay between length complexity and other notions of complexity. Finally, we show experimentally that simple next-token predictors, such as linear networks and shallow Multi-Layer Perceptrons (MLPs), display non-trivial performance on text generation and arithmetic tasks. Our results demonstrate that the power of today's LLMs can be attributed, to a great extent, to the auto-regressive next-token training scheme, and not necessarily to a particular choice of architecture.\footnote{Code for the experiments is available at: \url{https://github.com/emalach/LinearLM}}
\end{abstract}

\section{Introduction}

Large language models have achieved tremendous progress in various NLP tasks, such as machine translation, logical reasoning, coding and natural language understanding. These models, like GPT-3, GPT-4 and LaMDA \citep{brown2020language, OpenAI2023GPT4TR, thoppilan2022lamda}, are trained on massive amounts of text data and learn to generate coherent and contextually relevant responses to input prompts. Amazingly, such language models are mostly trained with a single objective: predicting the next token. While this objective seems extremely simplistic, auto-regressive next-token predictors trained on rich enough data are able to solve strikingly complex tasks \citep{bubeck2023sparks}. This raises the question of whether such next-token predictors are merely ``glorified'' autocomplete models, which happened to memorize the entire internet, or are they truly performing novel logical reasoning. To this end, it has been shown that the ability of language models to compute complex functions can be greatly enhanced by using Chain-of-Thought (CoT) and scratchpad techniques \citep{wei2022chain, kojima2022large, lightman2023let,nye2021show}, allowing the models to perform unrestricted intermediate computations before arriving at a final answer.

In this work, we introduce a theoretical framework for studying auto-regressive next-token predictors. We demonstrate that much of the power of today's language models in logical reasoning can be attributed to the nature of the auto-regressive learning, and not to a particular choice of architecture. We show theoretically that very simple models trained to only predict the next token in an auto-regressive fashion can be used to solve extremely complex tasks when utilizing CoT techniques. In particular, we show that even linear predictors---models where the next-token probability is a linear function of the input sequence---are already powerful enough to compute \emph{any Turing computable function}. The main theoretical result in the paper is captured in the following informal statement:
\begin{theorem}[informal]
\label{thm:informal}
    For any function $f$  that can be \emph{efficiently} computed using a Turing machine, there exists a dataset $D$ such that training a (linear) next-token predictor on $D$ results in a predictor that approximates $f$.
\end{theorem}

That is, any computer program or intelligent agent that can be simulated by a computer, can be learned, given the right dataset, by a simple next-token predictor.

To understand the power of auto-regressive learning, observe that a result equivalent to Theorem \ref{thm:informal} is not possible in the classical supervised learning setting, where the learner is given access only to the input sequence and the target label. It is well-known that no learning algorithm can efficiently learn the class of all (efficient) Turing computable functions \citep{valiant1984theory}, given only the input and the output of the function (without access to intermediate supervision). In fact, in classical supervised learning, there are only a few function classes that are known to be \emph{efficiently learnable}---function classes for which there exists a learning algorithm that can efficiently recover the target function given a labeled dataset. Learnable function classes are known to have fundamental limitations to their computational capacity. For example, the class of linear predictors is efficiently learnable in many settings, e.g. using the Perceptron algorithm \citep{rosenblatt1958perceptron}. However, a famous result in \cite{minsky2017perceptrons} shows that linear predictors cannot compute simple functions such as the XOR function. Auto-regressive learning, however, presents a striking difference. While linear next-token predictors are still \emph{efficiently learnable} using simple algorithms such as SGD, their computational capacity greatly surpasses the capacity of their \emph{classical} counterparts. Since auto-regressive inference introduces a sampling function\footnote{In our analysis we focus on the zero-temperature/argmax sampling, which acts as an explicit non-linearity. } after each step, it allows linear next-token predictors to compute non-linear functions. As implied by Theorem \ref{thm:informal}, linear next-token predictors can implement practically any target function of interest.

While next-token predictors have the capacity to generate highly proficient learners, this does not come without a cost. One significant expense is the requirement to provide the learning model with potentially long sequences of tokens that detail the internal computations of the target. This requirement can be resource-intensive and often impractical. As such, it prompts the introduction of a new measure of learning complexity, analogous to sample complexity or run-time complexity: the \emph{length complexity}. This type of complexity measures the quantity of intermediate tokens in a CoT necessary for the model to learn a particular concept class. We explore this complexity in the context of the parity learning problem, an extension of the XOR problem that is known to be computationally hard to learn in some settings. We demonstrate how traditional forms of complexity, such as sample or run-time complexity, can be traded off with length complexity when learning parities. Specifically, we show that an \emph{increase} in the complexity of the hypothesis class---and therefore in sample or computational complexity---leads to a \emph{decrease} in length complexity. This opens up a new path for the theoretical investigation of auto-regressive learning, by studying the interplay between these different complexity measures.

To substantiate our theoretical results, we experimentally illustrate the power of auto-regressive learning in enhancing the performance of simple models. We train a linear next-token prediction network on the TinyStories dataset \citep{eldan2023tinystories}, a collection of short stories composed of simple words. We observe that linear models, once trained on this dataset, frequently generate plausible and grammatically sound stories. Next, we demonstrate that a shallow Multi-Layer Perceptron (MLP) with 775M parameters (no attention layers), can learn to correctly multiply two 4-digit numbers, given CoT data. Our MLP achieves comparable results to Goat, a 7B-parameter transformer trained to solve arithmetic tasks \citep{liu2023goat}.

\begin{remark}
    We use the term \emph{chain-of-thought} to refer to situations where a language model, when given some input question, outputs a sequence of intermediate steps before arriving at the final answer. In general, this behavior may arise due to training on data with chain-of-thought demonstrations, or by prompting the model to ``think step by step''. Our theoretical results imply that linear auto-regressive models can compute complex functions using a chain of intermediate calculations, and that they learn to do so when trained on data with such chain-of-thought sequences. The study of chain-of-thought ``prompting'' of language models trained on a large corpus of data is beyond the scope of this paper.
\end{remark}


\section{Related Work}

\paragraph{CoT Reasoning} The proposition of supervising intermediate logical steps as an effective approach for problem-solving is well established, predating the advent of Transformer models. The technique was found to be particularly beneficial in solving arithmetic problems \citep{roy2016solving}. This idea became very popular with the introduction of the Chain-of-Thought (CoT) approach, where models are prompted to elucidate their thought process prior to yielding a final outcome \citep{wei2022chain, kojima2022large, lightman2023let}. Recent developments have further demonstrated the efficacy of the CoT method in the training of smaller student models \citep{li2023symbolic, li2022explanations, magister2022teaching}. Another method that bears similarity to CoT is the ``scratchpad'' technique, which allows models to record intermediate computations that subsequently aid in deriving the final answer \citep{nye2021show}. Such techniques have been shown to enhance performance across a variety of logical reasoning and arithmetic tasks. The research presented in this paper aims to contribute to the theoretical understanding of CoT reasoning in auto-regressive models. Our work illustrates how the employment of CoT can significantly amplify the capabilities of simple models. Furthermore, we introduce a novel complexity measure, the \emph{length complexity}, that allows us to study the influence of the length of the intermediate sequence of tokens within CoT on the difficulty of the learning problem.

\paragraph{Language Models for Arithmetic Tasks} Leveraging large language models to tackle mathematical reasoning and arithmetic tasks has gained significant interest, a trend that is discussed at length in a recent survey \citep{lu2022survey}. While these models have demonstrated a promising capacity for solving an array of mathematical problems, they often encounter difficulties in executing straightforward arithmetic operations, such as the multiplication and addition of large numbers \citep{nogueira2021investigating, qian2022limitations}. Previous studies have suggested that the efficiency of language models in arithmetic tasks can be dramatically enhanced by structuring them to perform calculations using an algorithmic pipeline, facilitating step-by-step execution \citep{muffo2023evaluating}. A notable contribution in this realm is the recent work by \cite{liu2023goat}, where they fine-tuned a moderately sized (7B-parameter) transformer employing the CoT method to perform complex arithmetic operations, including the multiplication of large numbers---a challenge even for advanced models like GPT-4. A very recent work studies the ability of small transformers trained from scratch to solve arithmetic tasks \citep{lee2023teaching}. In our study, we further substantiate this claim by demonstrating that a small MLP, devoid of any attention mechanism, can match the performance of the transformer in \cite{liu2023goat} in 4-digit multiplication, provided that it receives appropriate intermediate supervision. This highlights that the capability of language models for arithmetic and mathematical reasoning is largely attributable to the CoT and next-token prediction techniques, rather than the specific architectural choice.

\paragraph{Beyond Transformers} Although the transformer architecture \citep{vaswani2017attention} currently stands as the leading approach in language modeling, it is noteworthy that a diverse range of other architectures have served this purpose over time. A notable instance is the application of Recurrent Neural Networks (RNNs) \citep{hochreiter1997long}, a model highly popular for language modeling only a few years back, due to its efficient and inherent sequence processing capabilities \citep{mikolov2010recurrent}. Furthermore, convolutions have also been explored for language modeling tasks \citep{dauphin2017language}. A work more related to our own leveraged linear dynamical systems to model text \citep{belanger2015linear}. Recent years have witnessed an emerging interest in substituting the attention layer of transformers, primarily due to its high computational cost, with simpler and more efficient alternatives. In this vein, the work of \cite{katharopoulos2020transformers} introduced the linear transformer, where the attention layer was replaced with a more computationally-friendly linear layer. Concurrently, \cite{zhai2021attention} advanced an Attention-Free Transformer. More recent advancements include the RWKV architecture \citep{peng2023rwkv}, a modern variant of the RNN architecture inspired by transformers, which exhibits competitive performance when trained on large datasets. Some studies have proposed the use of simpler MLP-based architectures as feasible alternatives to transformers \citep{tolstikhin2021mlp, liu2021pay}. Our work contributes to this ongoing discourse by conducting both theoretical and empirical investigations into the potential of very simple models, such as linear models and small MLPs, training them to solve complex tasks by leveraging the power of next-token auto-regressive learning.

\paragraph{Related Theoretical Work} Despite the rapid pace of practical advancements in the realm of language models and transformers, the theoretical underpinning remains comparatively unexplored. Early investigations have established the universality of transformers (i.e., their ability to emulate any Turing machine) given the incorporation of a recurrent module \citep{yun2019transformers, wei2022statistically}. More recently, it has been demonstrated that transformers can simulate universal computers when incorporated into an execution loop \citep{giannou2023looped}. The work of \cite{liu2022transformers} shows that Transformers can simulate Automata, which are equivalent to bounded-memory programs, using surprisingly few layers. Turing universality extends to other language modeling architectures, such as RNNs \cite{siegelmann1992computational}. A study by  \cite{edelman2022inductive} underscores the inductive biases of self-attention, demonstrating that bounded-norm Transformer networks can represent sparse functions with logarithmically scaling sample complexity.
The work of \cite{feng2023towards} theoretically demonstrates the importance of CoT for solving mathematical problems with transformers. 
Of particular relevance to our study is the work of \cite{wies2022sub}, which delves into how sub-task decomposition and the CoT technique can facilitate the learning of computationally challenging problems. Similarly to our study, \cite{wies2022sub} also explores parity learning with intermediate supervision and demonstrates that arbitrary Turing machines can be efficiently learned by language models trained with CoT. Our work extends these findings, introducing a theoretical framework that enables broader examination of auto-regressive learning. We show that even linear predictors can efficiently learn Turing computable functions. In addition, our results offer improved length complexity bounds for learning parities, indicating that parities can be learned using $O(\log n)$ intermediate tokens, a marked reduction from the $O(n)$ intermediate tokens in \cite{wies2022sub}.

\section{Theory}

The key principle in our theoretical results is the differentiation between ``classical'' supervised learning and Auto-Regressive (AR) learning. In supervised learning, there is a clear separation between the input and the label (or target). The learner gets a dataset of inputs with their labels, and needs to find a model that correctly predicts the label of a new input example. While supervised learning tasks can sometimes be easy (e.g., when the label is given by a linear function of the input features), this task becomes very hard, or even impossible, when the function used for generating the labels requires a complex computational process \citep{valiant1984theory}. This hardness stems from the fact that the internal computation is not available to the learner, who only observes the input and the corresponding final output.

In AR learning, on the other hand, the situation is different. AR learners get a sequence of tokens, and treat every token both as an input (for predicting future tokens) and as a label (for sequences of previous tokens). Coupling AR learning with the CoT technique results in a learning paradigm where the internal computations required for reaching the final answer become available to the learner both as inputs and as \emph{labels}. This naturally allows supervision on intermediate steps in the computation/reasoning process, which greatly simplifies the learning task.

In the following sections we detail our theoretical results. In Section \ref{sec:learnability} we formally define the framework of AR Learning and Learnability, in an analogous way to classical PAC Learning. We then show how PAC Learnable hypothesis classes can be used for constructing AR Learnable classes, and discuss the special case of linear classes (which are known to be efficiently PAC Learnable). In Section \ref{sec:approximation} we discuss approximation results, namely understanding what types of function a given AR model can compute. To this end, we consider the function computed by the model to be the function mapping the input tokens to the final token(s), allowing the model to arbitrarily use internal computations in a chain-of-thought manner. Following this, we show that even linear AR models can compute very complex functions, for example emulating arbitrary Turing machines. Finally, in Section \ref{sec:length_complexity} we introduce \emph{length complexity}, which measures how many intermediate tokens are required in order to learn to compute a given function. We show that using more intermediate tokens, i.e. increasing the length complexity, can reduce time/sample complexity, and vice-versa.

\begin{remark}
It is important to note that in AR learning there is a crucial difference between the ``training mode'' and ``inference mode'' of the model. During training we use ``teacher forcing'': training the model to predict the next token given an input \emph{from the ground truth data}. During inference, however, we feed the model with some input and let it generate auto-regressively, predicting the next token given the sequence of tokens that were generated by the trained model itself. In our theoretical analysis we prove results in both the training mode setting (e.g., Theorem \ref{thm:ar_learnability}) and the inference mode (Theorem \ref{thm:boolean_circuit}). These results, while seeming independent, can be viewed as two parts of the same proof, giving guarantees on the inference-time behavior of some predictor trained with teacher forcing.
\end{remark}

\subsection{Learnability Results}
\label{sec:learnability}

Let $\bbD$ be a finite set of tokens, let $\cx = \bbD^n$ be the space of contexts of $n$ tokens, and let $\cz = \bbD^*$ be a space of strings of tokens. For some $t$, we denote $\cz_t = \bbD^t$. An Auto-Regressive (AR) function $h$ is a mapping $\cx \times \cz \to \bbD$ (we assume a deterministic function).
An AR hypothesis class $\ch$ is a set of AR functions. Fix some $T \in \naturals$.\footnote{In Section \ref{sec:length_complexity} we study how the choice of $T$ affects the complexity of the learning problem, but for now we treat $T$ as a fixed parameter of the learning problem.} For some distribution $\cd$ over $\cx \times \cz_T$, we say that $\cd$ is \emph{realizable} by the AR class $\ch$ if there exists a function $h \in \ch$ such, with probability $1$ over $(\x, \z) \sim \cd$, we have $h(\x, \z_{<t}) = z_{t}$ for all $t \le T$ (where $\z_{<t}$ denotes the first $t-1$ coordinates of $\z$). In other words, the pair $(\x, \z)$ is realizable by $h$ if $h$ accurately predicts the next token for all prefixes $\z_{<t}$ of $\z$. We now define Learnability in the AR framework:

\begin{definition}
    We say that $\ch$ is \emph{AR Learnable} if there exists a function $m : (0,1)^2 \to \naturals$ and an algorithm such that for every $\epsilon,\delta \in (0,1)$ and distribution $\cd$ realizable by $\ch$, given a sample of size $m(\epsilon, \delta)$ from $\cd$, returns with probability (w.p.) $\ge 1-\delta$ a function $\hat{h} \in \ch$ s.t.
    $\Pr\left[\exists t \le T ~\mathrm{s.t.}~\hat{h}(\x, \z_{<t})\ne z_{t}\right] \le \epsilon$.
Furthermore, we say that $\ch$ is \underline{efficiently} \emph{AR Learnable} if it is \emph{AR Learnable} with an algorithm running in polynomial time.
\end{definition}

That is, a class $\ch$ is (efficiently) AR Learnable if there exists an (efficient) algorithm that finds, w.h.p., a next-token predictor with low error.

We now show that hypothesis classes that are learnable in the classical sense (i.e., by supervised learning), naturally induce hypothesis classes that are AR Learnable. Let $\ch$ be some AR hypothesis class. We assume that $\ch$ can be decomposed into ``standard'' hypothesis classes in the following sense. Let $\{\ch_t\}_{t=1}^\infty$ be a sequence of classes, where $\ch_t$ is a class of functions $\cx \times \cz_{t-1} \mapsto \bbD$. We assume that $\ch = \ch_1 \times \ch_2 \times \dots$. Namely, we associate every $h \in \ch$ with a sequence $(h_1, h_2, \dots)$, where $h_i \in \ch_i$, s.t. for every $\x \in \cx$ and $\z \in \cz_{t-1}$ we have $h(\x, \z_{<t}) = h_t(\x, \z_{<t})$. While we define $\ch$ on arbitrarily long sequences, when we study learnability we limit ourselves to discussing sequences of length at most $T$. In particular, we can assume $\ch = \ch_1 \times \dots \times \ch_T$. The following result shows that PAC Learnability of the underlying hypothesis classes (as defined e.g. in \cite{shalev2014understanding}) implies AR Learnability of the class $\ch$:

\begin{theorem}
    \label{thm:ar_learnability}
    If $\ch_1, \dots, \ch_T$ are (efficiently) \emph{PAC Learnable} with sample complexity $m(\epsilon, \delta)$, then $\ch = \ch_1 \times \dots \times \ch_T$ is (efficiently) \emph{AR Learnable} with sample complexity $m(\epsilon/T, \delta/T)$. 
\end{theorem}

The proof (in Appendix \ref{app:proofs}) is a simple reduction using the standard notion of PAC Learnability.

\textbf{Linear Decoder} 

From Theorem \ref{thm:ar_learnability}, efficiently learnable classes induce classes that are efficiently learnable in the Auto-Regressive setting. 
For example, by letting $\ch_t$ be a class of linear functions, we can use known results on learning linear classifiers to show that the induced AR hypothesis class is efficiently learnable. We define the linear AR hypothesis class as follows.

\begin{definition}
    Let $\psi : \bbD \to \reals^d$ be some embedding of the dictionary. With some abuse of notations, for $\z \in \bbD^t$ we define $\psi(\z) = [\psi(z_1), \dots, \psi(z_t)] \in \reals^{d \times t}$. Fix some $t$, let $\W \in \reals^{\bbD \times d \times (n+t)}$, and for all $\x \in \cx$ and $\z \in \cz_t$ define $h_{\W}(\x,\z) = \arg \max_{D \in \bbD}\inner{W_{D},\psi([\x,\z])}$. Denote the function class of all linear predictors $\ch^\Lin_t = \{h_\W ~:~ \W \in \reals^{\bbD \times d \times (n+t)}\}$.
\end{definition}
Observe that the class $\ch^\Lin_t$ is PAC-learnable in polynomial time. Under some margin conditions and using a convex surrogate loss function, this class is in fact learnable using SGD \citep{shalev2014understanding}. Therefore, for the linear AR hypothesis class $\ch^\Lin = \ch^\Lin_1 \times \dots \times \ch^\Lin_T$, we get that $\ch^\Lin$ is efficiently learnable in the Auto-Regressive setting.

\subsection{Approximation Results}
\label{sec:approximation}

We showed that when the AR hypothesis class $\ch$ is induced from a sequence of (efficiently) learnable hypothesis classes, then $\ch$ is also (efficiently) AR learnable. In particular, $\ch^\Lin$ is efficiently AR learnable, as a product of linear classes. We now show that while learnability transfers from the classical setting to the AR setting, in AR learning we can get much stronger \emph{approximation} guarantees. In fact, while linear classes are relatively limited in the standard setting, we show that the linear AR class $\ch^\Lin$ is extremely powerful. Namely, we show that linear AR functions can efficiently approximate any Turing computable function.
 
We first need a proper definition of what are the functions that AR hypotheses ``compute''.
For some AR hypothesis $h$, define the output of the auto-regression process at time $t$ to be $h^{(t)}(\x)$, defined recursively by: $h^{(1)}(\x) = h(\x, \emptyset)$, and $h^{(t)}(\x) = h\left(\x, \left(h^{(1)}(\x), \dots, h^{(t-1)}(\x)\right)\right)$.
For now, we focus on AR hypotheses that are evaluated for $T$ steps, for some fixed $T \in \naturals$. In Section \ref{sec:length_complexity} we discuss how the choice of $T$ (length complexity) interacts with different measures of complexity.
We define the function computed (approximated) by $h$ as follows:
\begin{definition}
    \label{def:target_approx}
    Fix some target $f : \bbD^n \to \bbD$ and some AR hypothesis $h$. Then, we say that $h$ \underline{computes} $f$, if for every input $\x \in \bbD^n$ we have $h^{(T)}(\x) = f(\x)$. Additionally, for some distribution $\cd$ over $\bbD^n$, we say that $h$ \underline{$\epsilon$-approximates} $f$ w.r.t. $\cd$, if $\Pr_\cd\left[h^{(T)}(\x) \ne f(\x)\right] \le \epsilon$.    
\end{definition}

In other words, we say that $h$ computes $f$ if after running auto-regression for $T$ steps, it outputs a value that agrees with $f$. Note that we ignore all the intermediate outputs of $h$ and observe only the final output. This is in alignment with common practice, where we let language models use arbitrarily long chain-of-thought/scratchpad before arriving at the final answer\footnote{Here we assume that $f$ outputs a single token in $\bbD$, and therefore observe only the last token produced by the auto-regression. However, we note that this can be extended to the case where $f$ outputs multiple tokens, and we observe a sequence of tokens at the end of the auto-regression.}.

Next, we show that if some AR class $\ch$ is learnable, then auto-regressive learning of distributions realizable by $h \in \ch$ returns an approximator for the function computed by $h$:

\begin{theorem}
    \label{thm:app_learning}
    Assume that $\ch$ is (efficiently) AR Learnable with sample complexity $m(\epsilon, \delta)$. Then, there is an (efficient) algorithm s.t. for any $\epsilon,\delta$ and distribution $\cd$ realizable by some $h \in \ch$, given a sample of size $m(\epsilon, \delta)$, returns w.p. $\ge 1-\delta$ a function $\hat{h}$ s.t. $\hat{h}^{(T)}$ $\epsilon$-approximate $h^{(T)}$ w.r.t. $\cd$.
\end{theorem}

The proof follows by induction from the definitions (see Appendix \ref{app:proofs}).
Theorem \ref{thm:app_learning} shows that using AR learning, we can learn to approximate the function computed by the underlying AR function $h$.

\textbf{Approximation Capacity of Linear Hypotheses}

We now limit ourselves to a dictionary with only two tokens $\bbD = \{0,1\}$, to be compatible with standard analysis of computations with Boolean inputs/outputs. We will show that linear AR functions can approximate a very large class of functions---namely, the class of \emph{linear threshold circuits}.

\begin{definition}
    A linear threshold function is a func. of the form $x \mapsto \sigma(\inner{\bw,x} + b)$ for $\sigma(x) = \ind_{x\ge0}$.
    A linear threshold circuit is a Boolean circuit where every gate computes a linear threshold function.
\end{definition}

The following result shows that linear AR functions can approximate any linear threshold circuit:

\begin{theorem}
    \label{thm:boolean_circuit}
    Assume that $f : \{0,1\}^n \to \{0,1\}$ can be computed by a linear threshold circuit with at most $T$ gates. Then, $f$ can be computed by a linear AR function $h \in \ch^\Lin$.
\end{theorem}

The proof of the above result uses the fact that a linear threshold function can be implemented using $\argmax$ over a linear function, in the case where $\bbD = \{0,1\}$ (full proof in Appendix \ref{app:proofs}).

We note that any Turing computable function can be computed by a linear threshold circuit of some size $T$ that scales polynomially with the runtime of the Turing machine (see e.g. \cite{arora2009computational}). Therefore, we get that linear AR functions can compute any Turing computable function, with only polynomial blow-up in run-time. This leads to the following result:

\begin{corollary}
    For any function $f$ that is Turing computable in time $T(n)$, and for any distribution $\cd$ over inputs of size $n$, there exists a dataset of strings of tokens, each of size $\poly(T(n))$, s.t. training a linear AR model over this dataset efficiently recovers a function that approximates $f$ w.r.t. $\cd$.
\end{corollary}

To prove the above, we consider a dataset generated by a linear model simulating the target Turing machine which computes $f$.

\subsection{Length Complexity}
\label{sec:length_complexity}
We showed that even simple classes like linear AR predictors can approximate any Turing computable function. Since linear predictors can be learned efficiently, we get a learning scheme that can efficiently learn virtually any function of interest. This is in contrast with the standard supervised learning setting, where efficiently learnable function classes are typically very limited in their expressive power. However, we note that the complexity of learning did not magically ``disappear''. To make learning possible, we require that the learner has, during learning, access to a sequence of tokens representing the internal CoT generated by the target it aims to imitate. While the length of this sequence is still reasonable (polynomial in the problem parameters), acquiring data with such long sequences might be costly, or even impossible.

In this section we introduce \emph{length complexity}, a new notion of learning complexity that quantifies the number of intermediate tokens required for learning some concept class, i.e. the length of the CoT supervision provided to the model during training. The length complexity complements common complexity measures such as sample and run-time complexity, and we show that in some cases we can trade off sample/computational complexity for length complexity, and vice versa.

We begin with a formal definition of \emph{length complexity}.
Fix some distribution over $\bbD^n$, some AR hypothesis class $\ch$ and some target concept class $\cf$ of functions $\bbD^n \to \bbD$. The definition below extends Definition \ref{def:target_approx} to function classes, which allows an explicit discussion on length complexity.

\begin{definition}
    We say that $\ch$ computes $\cf$ with length complexity $T$, if $T$ is the minimal number satisfying that for every $f \in \cf$ there exists some $h \in \ch$ such that, for all $\x \in \bbD^n$ we have $h^{(T)}(\x) = f(\x)$. Additionally, we say that $\ch$ $\epsilon$-approximates $\cf$ with length complexity $T$ if for every $f \in \cf$ there exists some $h \in \ch$ s.t. $\Pr_\cd\left[h^{(T)}(\x) \ne f(\x)\right] \le \epsilon$.
\end{definition}

From Theorem \ref{thm:boolean_circuit} we get that the class of linear threshold circuits of size $T$ can be $\epsilon$-approximated using linear AR functions with \emph{length complexity} $T$. For small circuits this might not be an issue, but otherwise this dependence may be problematic. We expect that taking a richer AR hypothesis class $\ch$ would result in reduction of the length complexity.
In the rest of this section, we discuss the interplay between the choice of the AR hypothesis class and the different measures of complexity that it induces: sample complexity, computational complexity and length complexity.

\textbf{Length Complexity of Parities}

To demonstrate a concrete analysis of length complexity, we consider the well-studied problem of learning parities, a natural extension of the XOR problem \citep{minsky2017perceptrons}. In the parity learning problem, the inputs are sequences of $n$ bits, and the label is determined by the parity of the sum of an unknown subset of bits from the input. This problem is known to be computationally hard in some settings. For example, Statistical Query (SQ) algorithms and variants of gradient-descent need $\Omega(2^n)$ steps to solve the problem \citep{kearns1998efficient, shalev2017failures, abbe2018provable, malach2022hardness}, and it is hard to solve with limited memory \citep{raz2018fast}.

We now formally define the set of parity functions. Assume $\mathbb{D} = \{0,1\}$ (Boolean inputs).
For some subset $A \subseteq [n]$, define the parity function over $A$ by $\chi_A(\x) = \sum_{i \in A} x_i \mod 2$.
Let $\cp_n$ be the class of all parity functions, $\cp_n = \{\chi_A ~:~A \subseteq [n]\}$.
It is known that parities can be computed using $O(\log n)$ size linear threshold circuit \citep{kautz1961realization}. So, Theorem \ref{thm:boolean_circuit} implies that a linear AR model can compute any parity function with logarithmic length complexity:
\begin{theorem}
\label{thm:parity_lc}
The class $\cp_n$ can be computed using $\ch^\Lin$, with length complexity $O(\log n)$.
\end{theorem}

Since we showed that linear AR functions are efficiently learnable (Theorem \ref{thm:ar_learnability}), the above theorem implies that parities become efficiently learnable given $O(\log n)$ intermediate tokens. This is in contrast to the standard supervised learning setting, where linear functions cannot approximate parities \citep{daniely2020learning}. We note that a similar result on learning parities with intermediate tokens appears in \cite{wies2022sub}, but with $O(n)$ length complexity (instead of $O(\log n)$).

We next show that by taking more complex hypothesis classes we can reduce the length complexity of computing $\cp_n$. However, this comes at a cost of increasing either the sample or the computational complexity.
We define a sequence of AR classes of growing complexity for computing $\cp_n$. For every $k \le n$, let $\cp_{n,k}$ be the class of parities over subsets of size $\le k$, namely $\cp_{n,k} = \left\lbrace\chi_A ~:~ A \in \binom{[n]}{\le k}\right\rbrace$. 
The larger $n$ and $k$ are, the harder it is to learn $\cp_{n,k}$ (via supervised learning). In particular, there are known lower bounds on learning $\cp_{n,k}$ using Statistical Query (SQ) algorithms, a large family of algorithms that include variants of gradient-based learning algorithms \citep{blum2003noise}. Roughly speaking, learning $\cp_{n,k}$ using SQ algorithms requires run-time of $\binom{n}{\le k} = O((n/k)^k)$, and the sample complexity of $O(k \log n)$.
We define $\ch^{(k)} = \cp_{n,k} \times \cp_{n+1,k} \times \dots$, and show the following:
\begin{theorem}
    \label{thm:parity_lc2}
    $\ch^{(k)}$ can compute $\cp_n$ with length complexity $\Theta(n/k)$.
\end{theorem}

To prove the above result, we show that any parity over $n$ bits can be computed by constructing a ``tree'' of $k$-order parities, which reduces the length complexity by a factor of $k$ (see Appendix \ref{app:proofs}). This decrease in length complexity comes at the cost of increasing the computational complexity of the learning \emph{exponentially} with $k$ (for SQ algorithms and variants of GD). While the exact interplay between computational and length complexity depends on the learning problem, this result shows that sometimes decreasing the length complexity makes the problem computationally hard to learn. We believe that a fundamental understanding of the length complexity of different problems will allow us to better understand AR predictors. For example, discovering an intrinsic complexity measure for hypothesis classes (analogous to VC dimension or SQ dimension) that can be used to derive length complexity bounds is of particular interest. We leave such an investigation to future research.

\section{Experiments}
\begin{figure}
    \centering
    \includegraphics[scale=0.18]{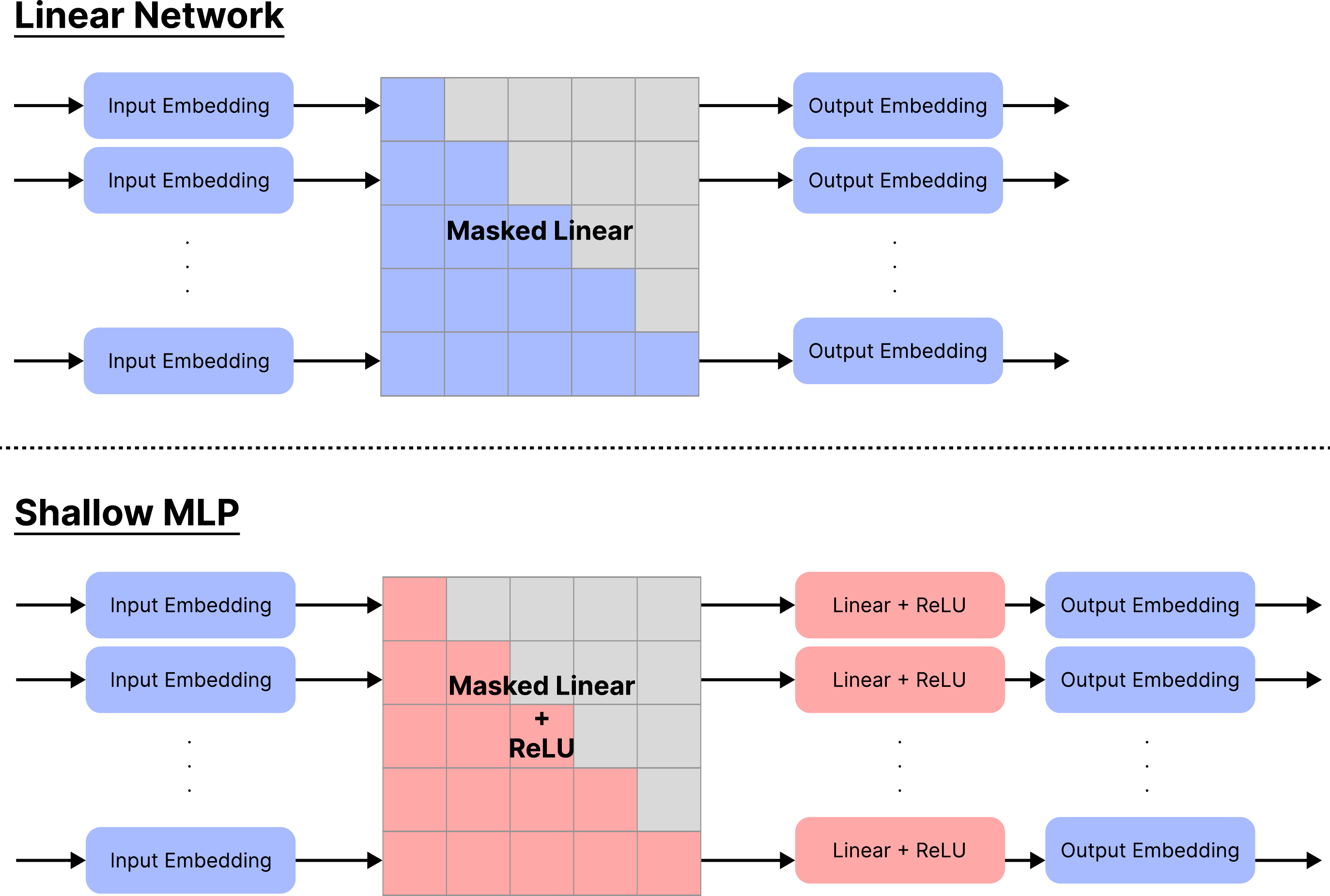}
    \caption{Illustration of the linear network and the MLP used in our experiments.}
    \label{fig:architecture}
    \vspace{-0.1in}
\end{figure}

We now turn to empirically validate our theoretical results, showing that very simple models perform surprisingly well when trained auto-regressively to perform next-token prediction. We start by training a simple linear model on a dataset of short stories, and then evaluate the performance of a small MLP on a task of arithmetic computations.

\vspace{-0.15in}
\subsection{Tiny Stories}
\begin{figure}
    \centering

\begin{mdframed}[style=textbox,userdefinedwidth=3in]
\scriptsize
\begin{tabularx}{\columnwidth}{lX}
\textbf{Prompt \#1:} & He was very scared \\
\hline
\textbf{Output:} & \emph{He was very scared. He cried out loud, but then he heard a voice. He was afraid of the dark monster. He ran to the bear. He had a pain in his hand. He screamed and ran away. He was never seen again. He was scared of the monster and he went on. He never\dots} \vspace{0.1in}\\

\textbf{Prompt \#2:} & Adam was hungry, so \\
\hline
\textbf{Output:} & \emph{Adam was hungry, so he wanted to eat the dessert. But he knew it was \rbg{not a good at his sister}. He thought, ``Maybe I can eat this food, but I don't want to share with you.''
} \vspace{0.1in}\\
\textbf{Prompt \#3:} & Alice was tired, so \\
\hline
\textbf{Output \#1:} &  \emph{Alice was tired, \gbg{so she decided to take a nap}. She put a blanket on the ground and started to cry. Then, she heard a noise. It was a nearby tree.
} \\
\textbf{Output \#2:} &\emph{Alice was tired, \rbg{so she decided to go on an} \rbg{adventure}. She hopped on the way to go home and look for her...
}
\end{tabularx}
\end{mdframed}

\vspace{0.1in}
\begin{center}
\begin{scriptsize}
\begin{sc}
\begin{tabular}{lccccr}
\toprule
Model & Grammar  & Creativity & Consistency & Plot \\
 & {\tiny GPT-4 / LT} & {\tiny GPT-4} & {\tiny GPT-4} & {\tiny GPT-4} \\
\midrule
TS-33M & 8.0 $\pm$ 0.8 / 62\% & 7.2 $\pm$ 0.5 & 7.0 $\pm$ 1.2 & 6.9 $\pm$ 0.8 \\
TS-1M & 6.9 $\pm$ 0.9 / 59\% & 6.7 $\pm$ 1.0 & 6.0 $\pm$ 1.5 & 5.6 $\pm$ 1.3\\
Linear & 6.3 $\pm$ 2.0 / 64\% & 6.2 $\pm$  1.8 & 5.9 $\pm$ 1.8 & 5.2 $\pm$ 1.8   \\
\bottomrule
\end{tabular}
\end{sc}
\end{scriptsize}
\end{center}

    \caption{\textbf{Top}: Example prompts and outputs for Linear model trained on TinyStories (grammatical/conceptual errors in red). \textbf{Bottom}: Comparison between Transformer-and Linear models, average grades from GPT-4.}
    \label{fig:tinystories}
    \vspace{-0.2in}
\end{figure}

We test the efficiency of linear AR models on the TinyStories dataset \citep{eldan2023tinystories}, a synthetic dataset of short stories containing simple words.
We train a linear model with context length of $T = 64$ on this dataset. The model has only three layers: 1) a standard (linear) embedding layer, mapping tokens into a vector of dimension $d = 256$; 2) a linear layer mapping $d \times T$ to $d \times T$ (using standard masking for next-token prediction during training); 3) an output embedding layer mapping vectors of dimension $d = 256$ back into the output space of all tokens (see Figure \ref{fig:architecture}). To allow next-token prediction training, we apply masking on the second linear layer, so that each output token only has access to previous tokens in the sequence. While the resulting classifier is linear, we note that this model is not exactly the linear AR model analyzed previously, as we allow sharing some parameters (namely, the input/output embedding parameters) across the different sequence positions. However, this is a close proxy to the idealized linear model. The model is optimized with the cross-entropy loss, using a softmax operation applied to the outputs. Altogether, the resulting model has roughly 162M active parameters. The model is trained for 5\nicefrac{1}{2} hours on a single A100 machine.

\begin{figure*}[ht]
{\small
    \centering
    \begin{minipage}{0.60\textwidth}
    \begin{mdframed}[style=textbox]
    \begin{tabularx}{\textwidth}{lp{5cm}}
    \textbf{Prompt:} & 1394$\times$8618= \\
    \hline
    \textbf{MLP:} & (4$\times$1+9$\times$10+3$\times$100+1$\times$1000)$\times$ \newline (8$\times$1+1$\times$10+6$\times$100+8$\times$1000)=\newline$\vdots$\newline
    \textbf{\gbg{12013492}} \\
    \hline
    \textbf{GPT-4:} & The multiplication of 1394 and 8618 equals \textbf{\gbg{12},\gbg{01}\rbg{4},\rbg{05}\gbg{2}}. \\
    \hline
    \textbf{Answer:} & \textbf{12013492}
    \end{tabularx}
    \end{mdframed}
    \label{fig:mul_output}
    \end{minipage}
    \begin{minipage}{0.38\textwidth}
    \begin{sc}
    \begin{tabular}{lcr}
        \toprule
        Model & Acc. (exact/per-digit) \\
        \midrule
        \textbf{MLP-775M} & 96.9\% / 99.5\% \\
        \hline
        \textbf{GPT-3.5} & 1.2\% / 61.9\% \\
        \textbf{GPT-4}* & 5.3\% / 61.8\% \\
        \textbf{Goat-7B}* & 96.9\% / 99.2\% \\
        \bottomrule
    \end{tabular}
    \end{sc}
    \label{tab:mul_table}
    \end{minipage}
        \caption{\textbf{Left:} Output of the MLP and GPT-4 on the 4-digit multiplication task (full output in Appendix \ref{app:figs}). \textbf{Right:} Performance of GPT vs. MLP model on the 4-digit multiplication task. *For GPT-4 and Goat-7B, we use the numbers as repored in \cite{liu2023goat}.}
        }
\end{figure*}

We use our model to generate story paragraphs, given some initial sentence. Similarly to \citet{eldan2023tinystories}, we use GPT-4 to grade 50 output examples based on grammar, creativity, consistency with the story’s beginning and
whether the plot makes sense  (see Appendix \ref{app:grading} for further details). We also evaluate the grammar of the generated text using the LanguageTool \cite{languagetool} grammar checker, and report the percentage of generations that had no grammatical errors (generating 5 outputs per story). While the results are certainly inferior in quality to transformer-based language models, we note that the linear predictor often does produce coherent and grammatically correct text. In Figure \ref{fig:tinystories} we show some example of prompts and the resulting output of the model. We emphasize that our goal is not to claim that linear models are better or even comparable to transformers, but rather to show that these extremely simple models achieve non-trivial language modeling capabilities when trained on high-quality data.



\subsection{Multiplication}

We now turn to demonstrate the power of next-token prediction with CoT reasoning for arithmetic tasks. We focus on the task of multiplying two 4-digit numbers, which has been shown to be challenging even for huge language models such as GPT-4 \citep{liu2023goat}. For this task, we train a simple Multi-Layered Perceptron (MLP) with four layers: 1) a standard (linear) embedding layer, from tokens to dimension $d=128$; 2) a linear layer with a ReLU activation, applied across all the context window, mapping the input of $d \times T$ to an output of $d \times T$ (where we use a context length of $T=307$); 3) a linear layer with a ReLU activation applied per token, mapping from $d$ to $d$; 4) a final output embedding, mapping back to the space of all tokens (see Figure \ref{fig:architecture}). Similarly to the linear network, we mask future positions in the second layer. We note that while this network has non-linearity (unlike the previous model), it is still very simple compared to standard transformer-based networks (e.g., we use no attention mechanism). Altogether, our MLP has 775M active parameters.

Recently, a paper by \cite{liu2023goat} instrodced Goat, a relatively small transformer fine-tuned from the LLaMA model that was able to outperform GPT-4 in various arithmetic tasks, when trained on data with intermediate calculations. We follow a similar procedure for training our model on 4-digit multiplication, with some key differences. First, we give more intermediate steps than in \cite{liu2023goat}, essentially unfolding the multiplication algorithm in the training sequences (see Figure \ref{fig:mul_output}). Second, we use a custom tokenization scheme, where we tokenize separately single digits ($1,2,3,\dots$), signs ($\times, +, =$) and also pairs of digits with multiplication sign ($1\times2$, $3 \times 5$, etc). This tokenization allows the model to quickly solve the single-digit multiplication task (by mapping pairs of multiplied digits to their product), which is a crucial tool in the multiplication algorithm. Finally, we also add zero-padding to some of the numbers, to get all strings to have the same length. 

We split all pairs of 4-digit numbers arbitrarily, use $75\%$ for training, and keep the rest for validation. The network is trained from scratch for 17 hours on a single A100 GPU, going over 100M sequences (307M tokens) sampled uniformly from the training set. In Table \ref{tab:mul_table} we compare the performance of our simple MLP (evaluated on 1000 validation examples) with GPT-3.5 (evaluated on the same examples), as well as to GPT-4 and Goat-7B on the same task (as reported in \cite{liu2023goat}). We report both accuracy of the exact match of the final answer, as well as accuracy of individual digits in the final number. We note that the performance of our MLP matches the performance of the much larger fine-tuned transformer in \cite{liu2023goat}\footnote{We also trained a small 70M transformer using our tokenization and CoT scheme. This transformer achieved only $72\%$ per-digit accuracy, far worse than the MLP or the 7B transformer of \citet{liu2023goat}. That said, it is possible that a bigger transformer can achieve more competitive results, but needs far more compute compared to our MLP.}, and outperforms both GPT-3.5 and GPT-4 on this task. This demonstrates again that a lot of the power of language models can be attributed to the next-token auto-regressive training on high-quality data, and not necessarily to a particular architectural choice.

\section{Discussion}

The emerging capabilities of large language models has triggered an ongoing debate about their potential and implications. Certain proponents assert that we are close to achieving Artificial General Intelligence (AGI), pointing to models such as GPT-4 which have already demonstrated perceived ``sparks of AGI'' \citep{bubeck2023sparks}. They argue that AGI is just a matter of scaling up—creating larger models, feeding them with more data, and increasing training time. In stark contrast, others dismiss these large models as merely sophisticated autocomplete systems, voicing concerns about their propensity to potentially absorb and perpetuate biased and harmful data \cite{bender2021dangers}.

While this debate is far from settled, we hope that our work sheds light on the theoretical possibilities inherent in training auto-regressive next-token predictors. Our findings indicate that, given suitable data, simple next-token predictors can be trained to effectively learn virtually any function of interest. Consequently, if there exists some computer program capable of realizing AGI, then it is theoretically plausible to attain AGI through training simple next-token predictors, given the appropriate data. Admittedly, these assertions, in their current form, are somewhat theoretical, with practical application requiring data composed of potentially very long sequences of intermediate computations. However, we show that by modifying the choice of the hypothesis class we can possibly shorten the required sequence length, making our results more realistic. Therefore, we believe that our research can contribute towards a better, more nuanced understanding of both the capabilities and constraints associated with next-token predictors.

\subsection*{Impact Statement}

This paper presents work whose goal is to advance the field of Machine Learning. There are many potential societal consequences of our work, none which we feel must be specifically highlighted here.

\bibliography{bib}
\bibliographystyle{icml2024}

\newpage
\onecolumn
\appendix

\section{Proofs}
\label{app:proofs}

\begin{proof}[Proof of Theorem \ref{thm:ar_learnability}]
    Let $\cd$ be some distribution over $\cx \times \cz_T$ realizable by $\ch$, and let $\cd_t$ be the distribution over $(\cx \times \cz_{t-1}) \times \bbD$, where we sample $(\x, \z) \sim \cd$ and observe $((\x,\z_{<t}),z_t)$. Therefore, $\cd_t$ is a labeled distribution realizable by $\ch_t$, and so we can use a learner for $\ch_t$ to find using $m(\epsilon/T,\delta/T)$ samples, with probability $1-\delta/T$, a hypothesis $\hat{h}_t$ s.t. $\Pr_\cd \left[\hat{h}_t(\x,\z_{<t}) \ne z_t\right] \le \epsilon/T$. Therefore, using the union bound, with probability at least $1-\delta$, we get:
    \[
        \Pr\left[\exists t \le T ~\mathrm{s.t.}~\hat{h}_t(\x, \z_{<t})\ne z_{t}\right] \le \sum_{t \le T} \Pr\left[\hat{h}_t(\x, \z_{<t})\ne z_{t}\right] \le \epsilon 
    \]  
\end{proof}

\begin{proof}[Proof of Theorem \ref{thm:app_learning}]

    By the definition of AR Learnability, we can find a hypothesis $\hat{h}$ s.t., with probability at least $1-\epsilon$ over $\x \sim \cd$, we get $$\hat{h}(\x,h^{(1)}(\x), \dots, h^{(t)}(\x)) = h(\x, h^{(1)}(\x), \dots, h^{(t)}(\x))$$ for all $t$. So, for such $\x$ we get
    \[
        \hat{h}^{(1)}(\x) = \hat{h}(\x, \emptyset) = h(\x, \emptyset) = h^{(1)}(\x)
    \]
    and by induction:
    \begin{align*}
        \hat{h}^{(t)}(\x) &= \hat{h}(\x,\hat{h}^{(1)}(\x), \dots, \hat{h}^{(t-1)}(\x)) \\
        &= \hat{h}(\x,h^{(1)}(\x), \dots, h^{(t-1)}(\x)) \\
        &= h(\x,h^{(1)}(\x), \dots, h^{(t-1)}(\x)) = h^{(t)}(\x)
    \end{align*}
\end{proof}

\begin{proof}[Proof of Theorem \ref{thm:boolean_circuit}]
    Let $f$ be some target circuit, and we define the depth of some gate in the circuit to be the maximal number of nodes in a path connecting the gate to some input variable. We sort the gates in the circuit by their depth, and let $f^{(1)}, \dots, f^{(T)}$ be the functions computed by the gates in the circuit (where $f^{(T)} = f$ is the output function). Observe that every gate $f^{(t)}$ can be computed by the argmax of a linear function of the inputs and previous gates, and therefore we can define some linear hypothesis $h$ s.t. $h(\x, f^{(1)}(\x), \dots, f^{(t-1)}) = f^{(t)}(\x)$. By induction, we get that for every $t$ we have $h^{(t)} = f^{(t)}$ and therefore the required follows.
\end{proof}

\begin{proof}[Proof of Theorem \ref{thm:parity_lc2}]
    To show that the length complexity is $O(n/k)$, observe that it is enough to construct a Boolean circuit of size $O(n/k)$, where every gate computes a parity over at most $k$ input bits (similarly to the proof of Theorem \ref{thm:boolean_circuit}). This circuit has the structure of a tree, where each node has in-degree at most $k$. It is easy to see that such a tree, with depth $\log_k(n)$ and $O(n/k)$ internal nodes can compute the parity over any subset of bits from the input.

    We now show that the length complexity is lower bounded by $\Omega(n/k)$. Assume, for the sake of contradiction, that $\cp_n$ can be computed with length complexity $T \le n/2k$, and particularly this implies that the parity over all input bits (namely, $\chi_{[n]}$) can be computed with $T \le n/2k$. Observe that, by the choice of the function class, at every step $t$ we have $h^{(t)}(\x) = \chi_{A_t}(\x)$ for some subset $A_t \subseteq [n]$. Additionally, at every step $t$, the size of $A_t$ can increase by at most $k$. Therefore, after $T \le n/2k$ steps, $h^{(T)}(\x) = \chi_{A_T}$ for some $A_T \subsetneq [n]$, and therefore $h^{(T)}(\x)$ does not compute (or even approximate) $\chi_{[n]}$.
\end{proof}

\newpage
\section{Additional Figures}
\label{app:figs}

\begin{figure}[h]
    \centering
    \begin{mdframed}[style=textbox]
    \begin{tabularx}{\textwidth}{lp{6cm}}
    \textbf{Prompt:} & 1394$\times$8618= \\
    \hline
    \textbf{Output (MLP):} & (4$\times$1+9$\times$10+3$\times$100+1$\times$1000)$\times$ \newline (8$\times$1+1$\times$10+6$\times$100+8$\times$1000)=\newline
    4$\times$8$\times$1$\times$1+4$\times$1$\times$1$\times$10+4$\times$6$\times$1$\times$100+\newline
    4$\times$8$\times$1$\times$1000+9$\times$8$\times$10$\times$1+9$\times$1$\times$10$\times$10+\newline
    9$\times$6$\times$10$\times$100+9$\times$8$\times$10$\times$1000+3$\times$8$\times$100$\times$1+\newline
    3$\times$1$\times$100$\times$10+3$\times$6$\times$100$\times$100+\newline
    3$\times$8$\times$100$\times$1000+1$\times$8$\times$1000$\times$1+\newline
    1$\times$1$\times$1000$\times$10+1$\times$6$\times$1000$\times$100+\newline
    1$\times$8$\times$1000$\times$1000=\newline
    32+040+2400+32000+720+0900+54000+\newline
    720000+2400+03000+180000+2400000+\newline
    08000+010000+0600000+08000000=\newline
    \textbf{\gbg{12013492}} \\
    \hline
    \textbf{Output (GPT-3.5):} & The multiplication of 1394 by 8618 is equal to \textbf{\gbg{12},\gbg{013},\rbg{6}\gbg{92}}.\\
    \textbf{Output (GPT-4):} & The multiplication of 1394 and 8618 equals \textbf{\gbg{12},\gbg{01}\rbg{4},\rbg{05}\gbg{2}}. \\
    \hline
    \textbf{Correct Answer:} & \textbf{12013492}
    \end{tabularx}
    \end{mdframed}
    \caption{Comparison between the output of our MLP, GPT-3.5 and GPT-4 on the 4-digit multiplication task.}

\end{figure}

\section{TinyStories GPT-4 Evaluation}
\label{app:grading}

To evaluate the models trained on the TinyStories dataset, we give all models a beginning of a sentence from a list of 50 sentences (generated by GPT-4). We then let the models complete the sentence, generating up to 64 tokens. We sample all models with temperature 0.5. We use the following prompt to GPT-4 to grade the quality of the produced text (this prompt is taken, with small modifications, from \cite{eldan2023tinystories}):

\textit{\textbf{The following exercise, the student is given a beginning of a sentence from a story. The student needs to complete it into a paragraph from the story.
The exercise tests the student´s language abilities and creativity. The symbol *** marks the separator between the
prescribed beginning and the student’s completion:}}

\textit{Once upon a time, in an ancient house,*** there lived a little girl named Mia. Mia loved to play with her toys and have fun with her friends. One day, Mia found a big box in her room. It was very pretty and shiny.}

\textit{\textbf{Please provide your general assessment about the part written by the student (the one after the *** symbol).
Is it gramatically correct? Is it consistent with the beginning of the story? Pay special attention to whether the
student manages to complete the sentence which is split in the middle by the separator ***.
The student’s completion of the story is mostly consistent with the beginning of the story. It maintains the focus on Lily
and her family, and the sentence split by the separator is completed correctly. However, the student’s addition does not fully
integrate the shiny decorations found in the attic, which were a significant part of the beginning.
The grammar is generally correct, but there are a few minor errors: ⟨list omitted⟩.
Overall, the student’s completion of the story demonstrates adequate language abilities and creativity, but could benefit from
better integration of the shiny decorations and minor grammar improvements.
Now, grade the student’s completion in terms of grammar, creativity, consistency with the story’s beginning and
whether the plot makes sense. Use numbers from 1 to 10.}}

\textit{\textbf{Output example:
Grammar: 8, Creativity: 7, Consistency:5, Plot: 6}}

We report the average score for all models.
Below is the list of sentences that we use:

\small
\begin{itemize}
\item\textit{``Once upon a time, in a colorful garden, there lived a tiny caterpillar named Charlie who...''}
\item\textit{``In the big, blue sky, Lucy the little bird was learning how to...''}
\item\textit{``Under the warm, shining sun, Benny the playful puppy found a...''}
\item\textit{``In the quiet, cozy barn, Millie the cow was dreaming about...''}
\item\textit{``On a bright, sunny morning, Oliver the curious kitten saw a...''}
\item\textit{``Deep in the green forest, Harry the little hedgehog was searching for...''}
\item\textit{``Near the sparkling river, Daisy the duck was making friends with...''}
\item\textit{``In the tall, whispering grass, Freddy the frog was hopping towards...''}
\item\textit{``Up in the fluffy white clouds, Peter the plane was flying over...''}
\item\textit{``Beneath the twinkling stars, Luna the owl was watching...''}
\item\textit{``Along the sandy beach, Sammy the crab was building a...''}
\item\textit{``In the busy, buzzing meadow, Bella the bee was collecting nectar from...''}
\item\textit{``On top of the snowy mountain, Eddie the eagle was soaring above...''}
\item\textit{``Inside the colorful coral reef, Wendy the fish was swimming with...''}
\item\textit{``At the edge of the mysterious jungle, Zoe the zebra was looking at...''}
\item\textit{``In the middle of the big city, Max the mouse was exploring...''}
\item\textit{``Under the bright rainbow, Ruby the rabbit was playing with...''}
\item\textit{``On the quiet farm, Ollie the ox was helping to...''}
\item\textit{``In the vast, open field, Ellie the elephant was trumpeting to...''}
\item\textit{``Next to the cool pond, Darcy the dragonfly was zooming around...''}
\item\textit{``In the old, wise tree, Timmy the squirrel was collecting nuts for...''}
\item\textit{``Around the busy bee hive, Polly the butterfly was fluttering near...''}
\item\textit{``Underneath the cozy blanket, Lily the lamb was dreaming of...''}
\item\textit{``Beside the gentle stream, Finley the fish was hiding from...''}
\item\textit{``In the dark, spooky cave, George the bat was hanging upside down and...''}
\item\textit{``Atop the ancient castle, Fiona the falcon was guarding...''}
\item\textit{``Within the enchanted forest, Greta the gnome was casting spells to...''}
\item\textit{``Behind the colorful rainbow, Nora the nymph was playing tricks on...''}
\item\textit{``Among the tall sunflowers, Sunny the sunbird was singing to...''}
\item\textit{``On the quiet moonlit night, Marvin the moth was flying towards...''}
\item\textit{``In the golden wheat field, Will the weasel was sneaking through...''}
\item\textit{``Along the sparkling coastline, Coral the seagull was searching for...''}
\item\textit{``Under the large oak tree, Oakley the owl was preparing for...''}
\item\textit{``In the middle of the pumpkin patch, Patty the pumpkin was waiting to...''}
\item\textit{``At the bottom of the deep ocean, Oscar the octopus was discovering...''}
\item\textit{``On the windy hilltop, Hannah the hawk was watching for...''}
\item\textit{``Inside the bustling anthill, Andy the ant was working hard to...''}
\item\textit{``Near the rosy apple tree, Amy the aardvark was sniffing around...''}
\item\textit{``At the edge of the shimmering lake, Leah the loon was diving for...''}
\item\textit{``In the shade of the big palm tree, Parker the parrot was chatting with...''}
\item\textit{``Beneath the golden sun, Gary the grasshopper was leaping towards...''}
\item\textit{``Along the bubbling brook, Brooke the beaver was building a...''}
\item\textit{``Under the cool, leafy canopy, Carl the caterpillar was munching on...''}
\item\textit{``In the soft, fluffy clouds, Claire the cloud was changing shapes to...''}
\item\textit{``On the bright, colorful rainbow, Roy the robin was hopping along...''}
\item\textit{``At the foot of the tall mountain, Monty the mountain goat was climbing up...''}
\item\textit{``Inside the warm, sunny greenhouse, Grace the gardener was planting...''}
\item\textit{``Near the quiet, sleepy village, Victor the vulture was soaring high above...''}
\item\textit{``Around the bustling city park, Paula the pigeon was pecking at...''}
\item\textit{``Under the starry night sky, Stella the starfish was dreaming about..."}
\end{itemize}

\end{document}